\documentclass{article}

\usepackage{microtype}
\usepackage{graphicx}
\usepackage{subfigure}
\usepackage{booktabs} 
\usepackage{color}
\usepackage{geometry}
\usepackage{capt-of}
\usepackage{sidecap}
\usepackage{wrapfig}
\usepackage{caption}
\usepackage{listings}
\usepackage{graphicx}
\usepackage{amsfonts}
\usepackage{amsmath}
\usepackage{amsthm}
\usepackage{subfigure}

\usepackage[utf8]{inputenc} 
\usepackage[T1]{fontenc}    
\usepackage{hyperref}       
\usepackage{url}            
\usepackage{booktabs}       
\usepackage{amsfonts}       
\usepackage{nicefrac}       
\usepackage{tabularx}
\usepackage{microtype}      

\usepackage{capt-of}
\usepackage{sidecap}
\usepackage{wrapfig}
\usepackage{caption}
\usepackage{listings}
\usepackage[html]{xcolor}
\definecolor{asterix_blue}{HTML}{4f87e6}
\definecolor{deepblue}{rgb}{0,0,0.5}
\definecolor{pythonbuiltincolor}{RGB}{228,86,73}
\definecolor{functioncolor}{RGB}{0, 84, 120}
\definecolor{objectcolor}{RGB}{80, 161, 79}
\definecolor{numbercolor}{RGB}{221, 136, 68}
\definecolor{deepgreen}{rgb}{0,0.5,0}
\definecolor{halfgray}{gray}{0.55}
\definecolor{ipython_frame}{RGB}{207, 207, 207}
\DeclareFixedFont{\ttb}{T1}{txtt}{bx}{n}{8.5} 
\DeclareFixedFont{\ttm}{T1}{txtt}{m}{n}{8.5} 

\hypersetup{
    colorlinks,
    linkcolor={red!70!black},
    citecolor={blue!50!black},
    urlcolor={blue!90!black}
}
\lstdefinelanguage[]{iPython}[]{python}{
    keepspaces=true,
    showspaces=false,
    showstringspaces=false,
    breaklines=true,
    basicstyle=\small\ttm,
    commentstyle=\color{cyan}\ttm,
    stringstyle=\color{deepgreen},
    keywordstyle = [1]{\ttb\color{pythonbuiltincolor}},
    keywordstyle = [2]{\ttb\color{functioncolor}},
    keywordstyle = [3]{\ttb\color{objectcolor}},
    morekeywords = [1]{class, def, self, return},
    morekeywords = [2]{__init__, forward},
    morekeywords = [3]{EquivalentEffectQNetwork},
    columns=fullflexible,
    literate=%
        {0}{{{\color{numbercolor}0}}}1
        {1}{{{\color{numbercolor}1}}}1
        {2}{{{\color{numbercolor}2}}}1
        {3}{{{\color{numbercolor}3}}}1
        {4}{{{\color{numbercolor}4}}}1
        {5}{{{\color{numbercolor}5}}}1
        {6}{{{\color{numbercolor}6}}}1
        {7}{{{\color{numbercolor}7}}}1
        {8}{{{\color{numbercolor}8}}}1
        {9}{{{\color{numbercolor}9}}}1
}
\usepackage{footnotehyper}
\definecolor{blue_forward}{HTML}{6520A4}
\definecolor{green_backward}{HTML}{417505}

\DeclareMathOperator*{\argmin}{arg\,min}
\usepackage{capt-of}
\usepackage{sidecap}
\usepackage{wrapfig}
\usepackage{caption}

\usepackage[accepted]{icml2024}

\usepackage{amsmath}
\usepackage{amssymb}
\usepackage{mathtools}
\usepackage{amsthm}

\usepackage[capitalize,noabbrev]{cleveref}

\theoremstyle{plain}
\newtheorem{theorem}{Theorem}[section]

\theoremstyle{definition}
\newtheorem{definition}[theorem]{Definition}
\newtheorem{assumption}[theorem]{Assumption}
\theoremstyle{remark}

\renewcommand{\algorithmicrequire}{\textbf{Inputs:}}
\renewcommand{\algorithmicensure}{\textbf{Output:}}

\usepackage[textsize=tiny]{todonotes}

\icmltitlerunning{Using Forwards-Backwards Models to Approximate MDP Homomorphisms}

\begin{document}

\twocolumn[
\icmltitle{Using Forwards-Backwards Models to Approximate MDP Homomorphisms}



\icmlsetsymbol{equal}{*}

\begin{icmlauthorlist}
\icmlauthor{Augustine N. Mavor-Parker}{UCLAI}
\icmlauthor{Matthew J. Sargent}{UCLAI}
\icmlauthor{Christian Pehle}{CSHL}
\icmlauthor{Andrea Banino}{deepmind}
\icmlauthor{Lewis D. Griffin}{UCLCS,equal}
\icmlauthor{Caswell Barry}{CDB,equal}
\end{icmlauthorlist}

\icmlaffiliation{UCLAI}{Centre for Artificial Intelligence, University College London, UK}
\icmlaffiliation{UCLCS}{Computer Science, University College London, UK}
\icmlaffiliation{CSHL}{Cold Spring Harbor Laboratory, NY, USA}
\icmlaffiliation{CDB}{Department of Cell and Developmental Biology, University College London, UK}
\icmlaffiliation{deepmind}{Google Deepmind, London, UK}

\icmlcorrespondingauthor{Augustine N. Mavor-Parker}{a.mavor-parker@cs.ucl.ac.uk}

\icmlkeywords{Machine Learning, ICML}

\vskip 0.3in
]
\begin{abstract}
Reinforcement learning agents must painstakingly learn through trial and error what sets of state-action pairs are value equivalent---requiring an often prohibitively large amount of environment experience. MDP homomorphisms have been proposed that reduce the MDP of an environment to an abstract MDP, enabling better sample efficiency. Consequently, impressive improvements have been achieved when a suitable homomorphism can be constructed a priori---usually by exploiting a practitioner's knowledge of environment symmetries. We propose a novel approach to constructing homomorphisms in discrete action spaces, which uses a learnt model of environment dynamics to infer which state-action pairs lead to the same state---which can reduce the size of the state-action space by a factor as large as the cardinality of the original action space. In MinAtar, we report an almost 4x improvement over a value-based off-policy baseline in the low sample limit, when averaging over all games and optimizers.
\end{abstract}

\section{Introduction}
Reinforcement learning (RL) has achieved superhuman performance on previously unsolvable benchmarks such as Go \citep{silver2016mastering} and Starcraft \citep{vinyals2019grandmaster}. However, the computational expense of RL hinders its deployment in the real world. RL can demand hundreds of millions of samples to learn a policy, either within an environment model \citep{hafner2020mastering} or by direct interaction \citep{badia2020agent57}. While sample efficiency is improving, there is a gap between the best performing agents and those that are sample efficient \citep{schwarzer2023bigger,  fan2023learnable}.
\begin{figure}
\centering
\includegraphics[width=0.5\textwidth, trim={2cm 2.89cm 0 0.95cm}, clip]{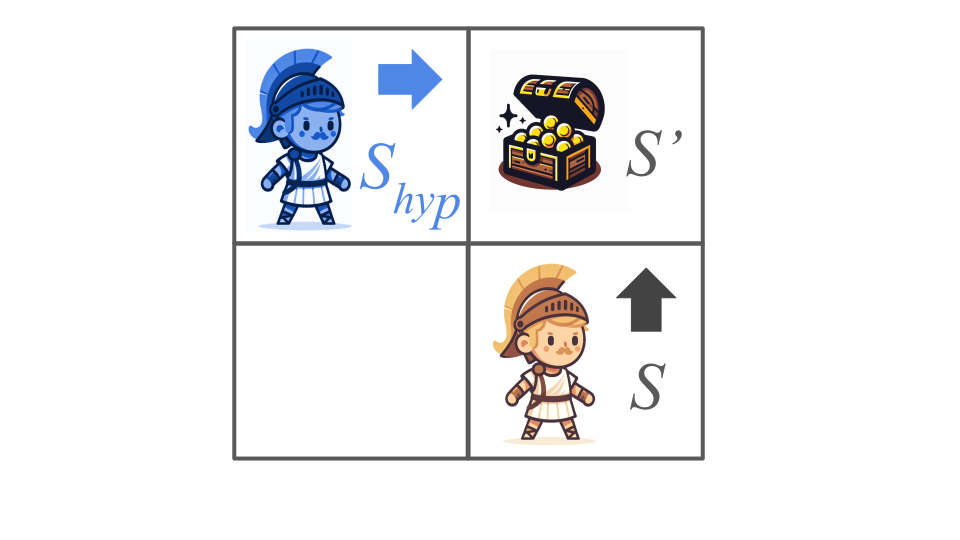}
\caption{We predict equivalent state-action pairs with dynamics models. In Asterix, an agent needs to collect gold. To predict the value of state $S$ and action \textit{move up}, we first predict forward in time to get $S'$. We then predict backwards assuming the previous action was the alternative \textit{\textcolor{asterix_blue}{move right}}, to obtain a hypothetical state-action pair for our Q-network (\textcolor{asterix_blue}{$S_{hyp}$}, \textcolor{asterix_blue}{\textit{move right}}). Icons generated with Dall-e 3 \cite{betker2023improving}.}
\label{fig:eea_intuition}
\end{figure}

Humans abstract away details about state-action pairs that do not effect their values. Consider the Atari game Asterix, where an agent collects gold while avoiding enemies in a 2D navigation task---a human gamer understands that collecting treasure when approaching from below has the same value as collecting the same treasure but approaching from the right (figure \ref{fig:eea_intuition}), which allows them to progress more quickly by generalising value predictions between equivalent state-action pairs. In this work, we identify this abstraction as a Markov decision process (MDP) homomorphism \citep{ravindran2001symmetries, van2020mdp}. MDP homomorphisms collapse equivalent state-actions in an observed MDP onto a smaller abstract state-action space in an abstract MDP \citep{van2020mdp}.

Given a mapping to an abstract MDP, policies can be learned efficiently in the smaller abstract space and then mapped back to the experienced MDP when interacting with the environment \citep{ravindran2001symmetries}. Previous works hard code homomorphisms into a policy \citep{van2020mdp} but developing approaches that learn homomorphic mappings online from experience is an active area of research \citep{rezaei2022continuous, keurti2023homomorphism}.

We present \textit{equivalent effect abstraction}, a method for constructing MDP homomorphisms from experience via a partial dynamics model of state transitions (without needing to predict rewards)---leveraging the fact that state-action pairs leading to the same next state frequently have equivalent values. We use this insight in off-policy value-based RL, which is the most sample efficient class of RL algorithms for discrete control \cite{schwarzer2023bigger}. Off-policy RL predicts separate values for different actions in a given state. We improve sample efficiency further by only requiring a value function to predict the value of a single action selected by the practitioner, which we call the \textit{hypothetical action}, in a separate state called a \textit{hypothetical state}. The hypothetical state is computed by first moving through a forwards dynamics model with a given state-action pair. Then, we move backwards through a model that predicts what hypothetical state would lead to the same next state as the original state-action pair, given the hypothetical action was selected. 

Consider again the Asterix example---moving to a given state has the same value whether your previous action was right, left, up or down. Thus, equipped with a dynamics model, we show that if the value of approaching a state from the right is known, one can also deduce the value of approaching the same state from the left, below, or above. Consequently, equivalent effect abstraction extrapolates value judgements between equivalent state-action pairs, reducing the amount of experience required to learn a policy. We make the following contributions.

\begin{enumerate}
    \item We develop a novel approach for learning a reduced MDP without knowledge of symmetries, and prove that, under certain assumptions, our mapping to the reduced MDP is a homomorphism.
    \item We confirm these theoretical results in an idealised tabular setting, where we show equivalent effect abstraction improves the planning efficiency of model-based RL and the sample efficiency of model-free RL.
    \item In the deep RL setting, we test our approach when the homomorphism is approximate due to imperfect learned models---we show an improvement upon standard baselines on Cartpole and Predator-Prey environments (which are MDP homomorphisms benchmarks from \citet{van2020mdp}).
    \item In the MinAtar suite \citep{young19minatar}, where the assumptions required for our theoretical results are further violated (especially relating to stochasticity), we find an almost 4x improvement over a value based baseline at 250k frames (when averaging all results over different optimizers and games). 
\end{enumerate}
We are not the first to exploit the equivalence of state action pairs that lead to the same next state. \citet{tesauro1995temporal} use the concept of afterstates (states immediately before an opponents response to an action \citep{antonoglou2021planning}) in their backgammon program. This is similar to our approach and has been scaled within MuZero to other board games \citep{antonoglou2021planning}, but in general afterstates have not found wider applicability. \citet{misra2020kinematic}'s kinematic inseparability is a related but distinct abstraction that uses contrastive learning to classify whether two states will lead to the same next state for the \textit{same} action. In contrast, our approach looks for state-action pairs (with \textit{different actions}) that lead to the same next state, which is a broader form of abstraction. Another difference is that our equivalent effect abstraction does not require states to be discrete, which makes it readily applicable to common RL environments where it is rare to reach the same state twice. Finally \citet{anand2015asap} defines a state-action equivalence that is more comprehensive than the equivalence defined by \citet{ravindran2001symmetries}, which is satisfied by various equivalent state-action pairs like those found in \citet{ravindran2001symmetries, van2020mdp, misra2020kinematic} as well as the state-action pairs satisfying the equivalent effect abstraction we present. However, they assume a tree of discrete states can be searched over for equivalences, which is challenging in deep RL.
\section{Equivalent Effect Abstraction}
\subsection{Preliminaries}\label{HMDP}
Using the definition from \citet{sutton2018reinforcement}, an MDP $\mathcal{M}$ can be described by a tuple $\left< \mathcal{S}, \mathcal{A}, \mathcal{P}, \mathcal{R}, \gamma\right>$ where $\mathcal{S}$ is the set of all states, $\mathcal{A}$ is the set of all actions, $\mathcal{R}=\mathbb{E}[ R_{t}|S_{t-1}=s, A_{t-1}=a ]$ is the reward function that determines the scalar reward received at each state, $\mathcal{P} = \text{Pr}\{S_{t}=s'|S_{t-1}=s, A_{t-1}=a\}$ is the transition function of the environment describing the probability of moving from state to another for a given action and $\gamma \in [0, 1]$ is the discount factor describing how much an agent should favour immediate rewards over those in future states. An agent interacts with an environment through its policy $\pi(a|s)=\text{Pr}\{A_{t}=a|S_{t}=s\}$ \citep{sutton2018reinforcement}, which maps the current state to a given action. An RL agent learns a policy that maximises the expected return $G_{t}$, equal to the sum of discounted future rewards $G_{t}=\mathbb{E}_{\pi}[\sum_{t=0}^{T}\gamma^{t}R_{t+1}]$, where $t$ is the timestep and $T$ is length of a learning episode \cite{sutton2018reinforcement}. 
\newline
\newline
This work focuses on improving the sample efficiency of value-based reinforcement learning algorithms. Value-based algorithms predict the expected returns of a given agent---often by learning Q-functions $Q_{\pi}(s,a)=\mathbb{E}[G_{t}|S_{t}=s, A_{t}=a]$ \cite{sutton2018reinforcement}. The amount of samples required to learn a Q-function depends, in part, on the size of the state action space for which it needs to accurately predict values. Equivalent effect abstraction aims to reduce the number of state-action pairs a value function needs to model by leveraging knowledge from learned transition models. Recent work in the model-based RL literature has used backwards model to gain more supervision for latent representation learning \cite{yu2021learning, yu2021playvirtual}. Here, we find another use for forwards-backwards models---to find state-action pairs with similar values.
\subsection{Transition Models}
Equivalent effect abstraction requires two transition models: a forward transition model 
$f_{\theta}: \mathcal{S} \times \mathcal{A} \rightarrow \mathcal{S}$, which is parameterised by $\theta \in \mathbb{R}^{D_{f}}$ where is $D_{f}$ is the number of parameters; and a backwards transition model $b_{\phi}: \mathcal{S} \times \mathcal{A} \rightarrow \mathcal{S}$, which is parameterised by $\phi \in \mathbb{R}^{D_{b}}$ where is $D_{b}$ is the number of parameters in the backwards model. These two models are trained to model the forwards and backwards dynamics of the environment. More explicitly, $f_{\theta}$ and $b_{\phi}$ are trained on collected experience tuples $\left< {s}_{t}, a_{t}, {s}_{t+1}\right>$ to minimize a supervised learning objective.
\begin{equation}
    \argmin_{\theta} \|f_{\theta}(s_{t}, a_{t}) - s_{t+1}\|^{2}_{2}
\end{equation}
\begin{equation}
    \argmin_{\phi} \|b_{\phi}({s}_{t+1}, a_{t}) - {s}_{t}\|^{2}_{2}
\end{equation}
In words, $f_{\theta}$ predicts what state the agent transitions into after taking action $a_{t}$ in state $s_{t}$, while $b_{\phi}$ predicts what state the agent has come from if it has just taken $a_{t}$ and ended up in $s_{t+1}$. Next, we explain how these models are used to extract a smaller MDP than the one experienced, by reducing the action space to an action we call the \textit{hypothetical action}. 
\begin{definition}
    The hypothetical action $a_{hyp}$ is a single action $a_{hyp}\in\mathcal{A}$ selected by the practitioner before training.
\end{definition}
When predicting values, equivalent effect abstraction maps state-action pairs into equivalent hypothetical state-action pairs using the transition models defined and a predetermined hypothetical action (see algorithm \ref{alg:example}). Before introducing the reduced MDP in section \ref{homomorphism_proof_section}, we first define the assumptions that our mapping relies on, which is then used to support our theoretical results. In section \ref{assumptions_discussion}, we discuss the limitations introduced by our assumptions. In section \ref{experiments}, we investigate empirically how equivalent effect abstraction behaves when these assumptions are relaxed. 
\subsection{Assumptions}
\begin{assumption}\label{ass:one}
    Transitions are deterministic, meaning for every state-action pair $(s, a)$, the transition dynamics are $\mathcal{P}(s, a, s')=1$ for a single next state $s'\in \mathcal{S}$ and $\mathcal{P}(s, a, s'')=0$, $\forall s'' \in \mathcal{S}$ where $s'' \neq s'$.
\end{assumption}
This assumption means that we can predict what state $s'$ a state-action pair $(s, a)$ will transition into with certainty. Conversely, it also means we can predict what state an agent could have just come from given a previous action $a$ and the current state $s'$. This is useful, because in the construction of our reduced MDP, we select one hypothetical action $a_{hyp}$ for our reduced action space and then we map all state action-pairs into this action's reference frame when predicting values. That is, we map a given state action pair $(s, a)$ to a hypothetical state-action pair $(s_{hyp}$, $a_{hyp})$ (see algorithm \ref{alg:example}). For this mapping to be generally computable, it is also required that the following assumption holds. 
\begin{assumption}\label{ass:three}
    There exists an action $a_{hyp}\in\mathcal{A}$, such that for all $s, s' \in \mathcal{S},$ $a \in \mathcal{A}$ there exists $s_{hyp} \in \mathcal{S}$ such that 
    $\mathcal{P}(s_{hyp}, a_{hyp}, s')=\mathcal{P}(s, a, s').$
\end{assumption}
This means for every state-action pair, at least one equivalent hypothetical state-action pair can be computed. This is often true but it is not guaranteed. For example, near borders in a grid world there is no way to travel left to a border state that has a border on its right. As a result, there is no mapping for a hypothetical action of move left if the current state sits to the right of a border. We address the consequences of this in section \ref{assumptions_discussion}.

\begin{assumption}\label{ass:two}
    Rewards only depend on the state the agent transitions into $\mathcal{R}(s_{hyp}, a_{hyp}, s')=\mathcal{R}(s,a, s')$.
\end{assumption}
State-action pairs that lead to the same state have equivalent rewards in the majority of RL testbeds such as DeepMind control \citep{tassa2018deepmind}, classic control \citep{brockman2016openai} and Atari \citep{bellemare2013arcade}. While common MDP formulations (including section \ref{HMDP}) define reward functions as dependent on the previous action, it is often possible to drop the dependence on the previous action. Nevertheless, we highlight that we assume no dependence on the previous action, so our approach is not suitable for optimal control settings where different actions have different costs \citep[p. 145]{liberzon2011calculus}.
\subsection{Equivalent effect abstraction provides a MDP homomorphism}\label{homomorphism_proof_section}
Given assumptions [\ref{ass:one}-\ref{ass:two}] and transition models, we can map an MDP $\mathcal{M}$ to a smaller hypothetical MDP $\mathcal{M}_{hyp}$.
\begin{definition}\label{definition:EEA}
The hypothetical MDP $\mathcal{M}_{hyp}$ induced by equivalent effect abstraction is a tuple $\mathcal{M}_{hyp}=\left<\mathcal{S}, \mathcal{A}_{hyp}, \mathcal{P}_{hyp}, \mathcal{R}_{hyp}, \gamma\right>$ where $\mathcal{S}_{hyp} = \mathcal{S}$, $\mathcal{A}_{hyp} = \{a_{hyp}\}$, $\mathcal{R}_{hyp}=\mathcal{R}$ and $\mathcal{P}_{hyp}(s, a, s') = \mathcal{P}(b_{\phi}(f_{\theta}(s, a), a_{hyp}), a_{hyp}, s')$, which simplifies to $\mathcal{P}_{hyp}(s, a, s')=\mathcal{P}(s_{hyp}, a_{hyp}, s')$.
\end{definition}
Equivalent effect abstraction maps all actions in the action space to one hypothetical action, while mapping a given state to an equivalent hypothetical states in the original state space (algorithm \ref{alg:example}). By finding an equivalent hypothetical state-action pair for each experienced state-action pair, we reduce the number of state-action pairs that an off-policy algorithm must learn to estimate. In the ideal case where all assumptions are satisfied, this reduces the state-action space a value function needs to model from $\mathcal{S}\times\mathcal{A}$ to $\mathcal{S}\times1$. 
\begin{definition}
An MDP homomorphism \citep{ravindran2004algebraic, van2020mdp, rezaei2022continuous} is a tuple of surjections $h=\left< \sigma_{a}, {g_{s}|\{(s, a) \in (\mathcal{S} \times \mathcal{A})}\}\right> $ that maps from one MDP $\mathcal{M}$ to another MDP $\bar{\mathcal{M}}$ where:
\begin{equation}\label{eqn:transition_equivalence}
    \bar{\mathcal{P}}(\sigma_{a}(s), g_{s}(a), \sigma_{a}(s'))=\sum_{s''\in [s']_{B_{h}|\mathcal{S}}}\mathcal{P}(s, a, s'')
\end{equation}
\begin{equation}\label{eqn:reward_equivalence}
    \bar{\mathcal{R}}(\sigma_{a}(s), g_{s}(a), \sigma_{a}(s'))=\mathcal{R}(s, a, s'')
\end{equation}
The homomorphism splits the state space $\mathcal{S}$ in subsets of states---$[s']_{B_{h}|\mathcal{S}}$ is the subset of states that are in the same subset as $s'$ \cite{ravindran2004algebraic, rezaei2022continuous}. Equations (\ref{eqn:transition_equivalence}) and (\ref{eqn:reward_equivalence}) show that a homomorphic map maintains the transition dynamics and reward functions of the experienced MDP within the constructed abstract MDP. 

Note that we have made an extension to the definition of \citet{ravindran2004algebraic}, as our definition uses state mappings $\sigma_{a}$ that are dependent on actions, not just state inputs. This naturally encompasses previous definitions of MDP homomorphisms as the mapping $\sigma_{a}$ can simply ignore input actions to be consistent with \citet{ravindran2004algebraic}. This has no effect on the optimal value equivalence property (theorem \ref{theorem:value_equivalence}), as the proof from \citet{ravindran2004algebraic} does not rely on $\sigma$ only being a function of states only. 
\end{definition}
\begin{theorem}\label{theorem:EEA_is_homomorphism}
    Given assumptions [\ref{ass:one}-\ref{ass:two}] and exact models of forwards and backwards dynamics, the MDP mapping provided by equivalent effect abstraction (definition \ref{definition:EEA}) is an MDP homomorphism. 
\end{theorem}
\begin{proof}
    proof in appendix \ref{proof:homomorphism}.
\end{proof}
As equivalent effect abstraction is an MDP homomorphism (theorem \ref{theorem:EEA_is_homomorphism}), it inherits the value equivalence property.
\begin{theorem}\label{theorem:value_equivalence}
    \citep{ravindran2004algebraic} The optimal action-value function $Q^{*}$ in an MDP implied by an MDP homomorphism is equal to the optimal action-value function in the full MDP.
    \begin{equation}
        Q^{*}(\sigma_{a}(s), g_{s}(a)) = Q^{*}(s, a) 
    \end{equation}
\end{theorem}
Consequently, an agent equipped with a homomorphic map can learned values in the abstract state-action space and the used these to predict values in the experienced state-action space \citep{ravindran2001symmetries}. Note that in addition to the assumptions discussed next, theorem \ref{theorem:EEA_is_homomorphism} states that equivalent effect abstraction requires exact models, meaning predicting transitions without any prediction error. This is usually not feasible, which motivates our empirical work in sections [\ref{cartpole_section}-\ref{minatar}] to see how our approach behaves when the model is imperfect and learned from experience.
\begin{figure*}[htp]            
\centering
\subfigure{\includegraphics[width=0.4\textwidth]{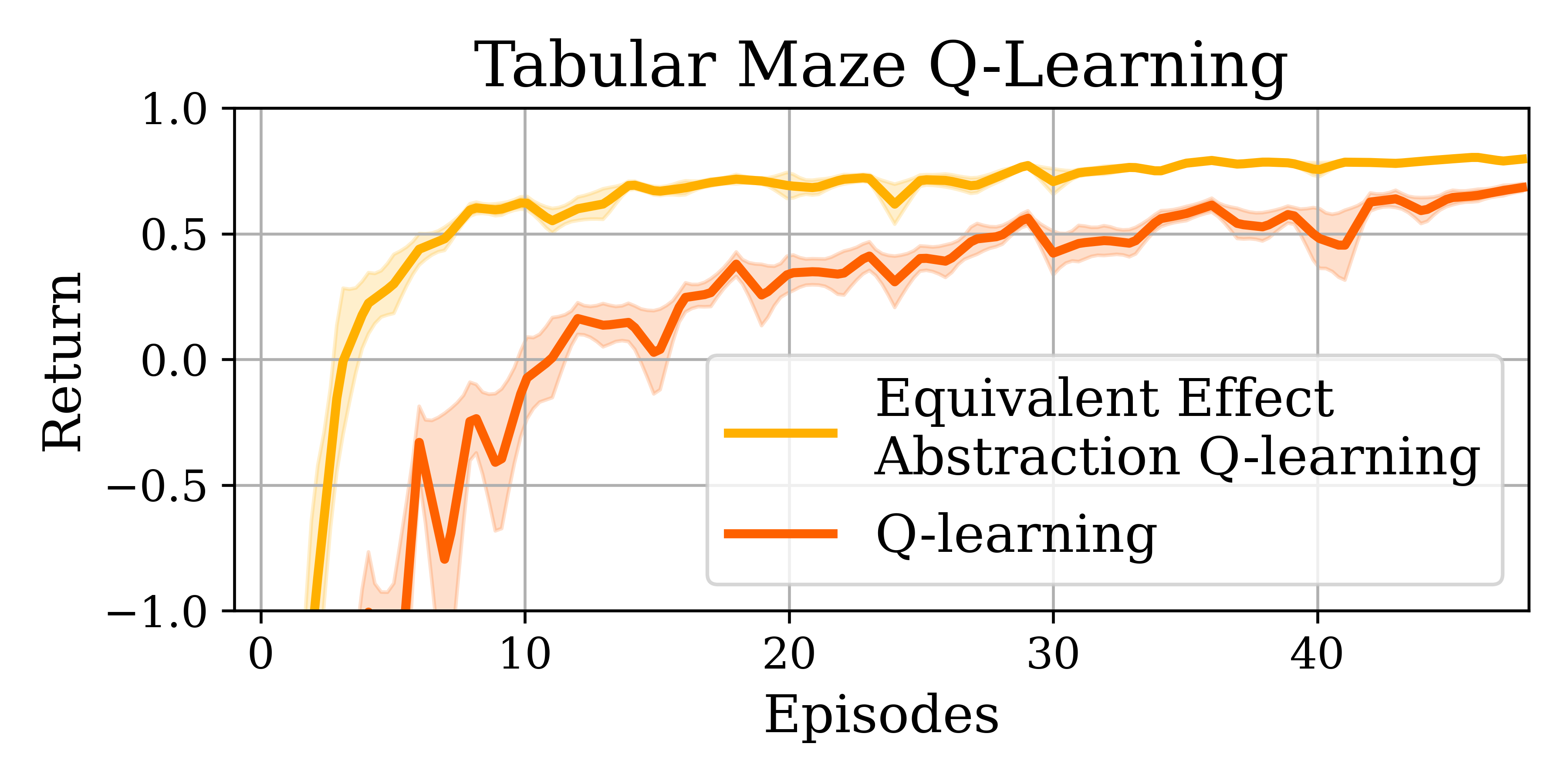}}\quad
     \subfigure{\includegraphics[width=0.4\textwidth]{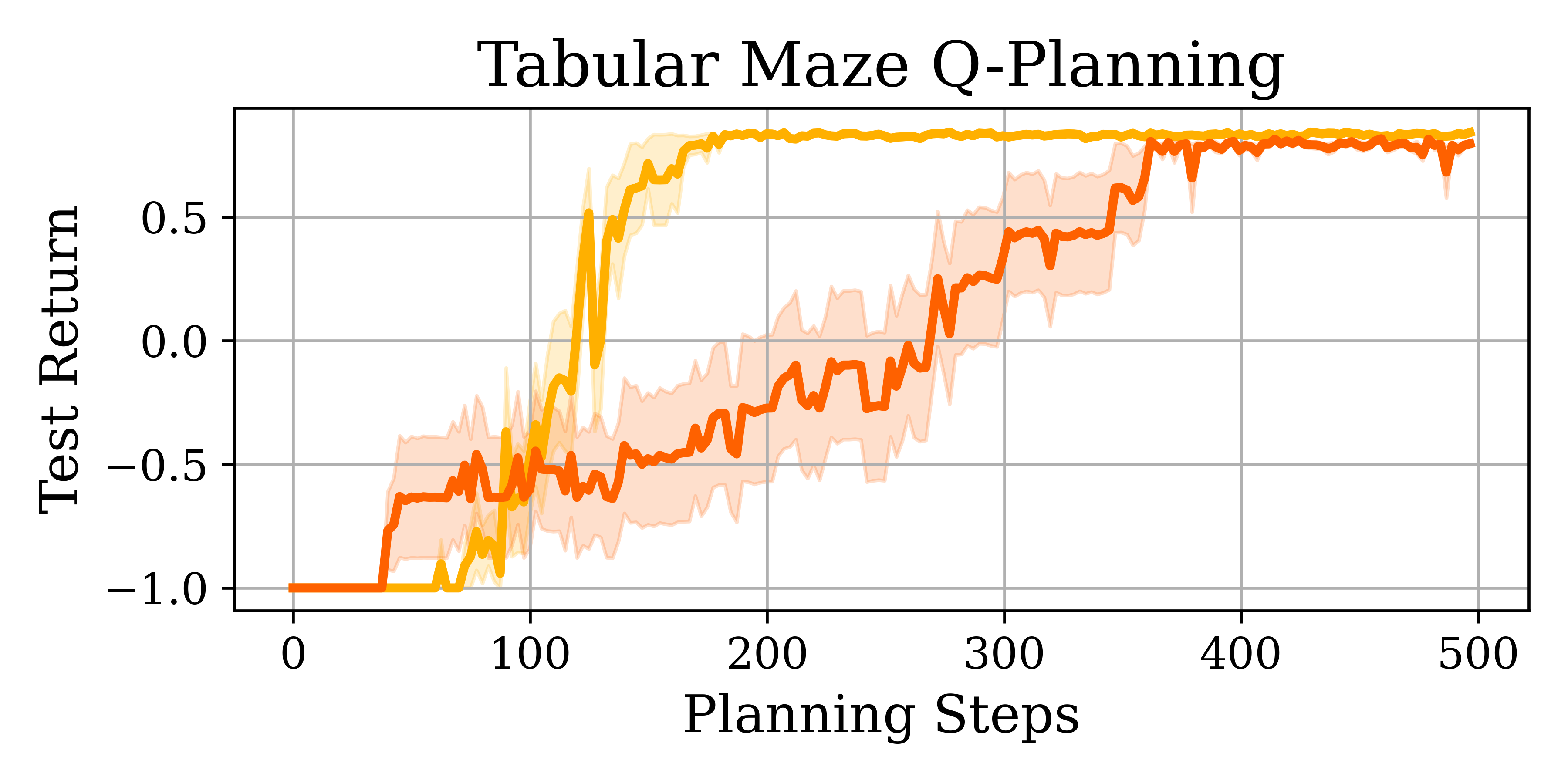}}    
     \caption{$\mathbf{(a),(b)}$ In a gridworld, equivalent effect abstraction improves sample efficiency. Equivalent effect abstraction improves both model-free Q-learning and model-based approaches Q-planning. 50 seeds are used for Q-learning while 10 seeds are used for Q-planning.}\label{gridworld_figure}
\end{figure*}
\begin{algorithm}[tb]
   \caption{Equivalent Effect Abstraction Value Prediction}
   \label{alg:example}
\begin{algorithmic}
   \REQUIRE Forward Model $f_{\theta}$, Backward Model $b_{\phi}$, List of $N$ environment actions $A_{N}$, hypothetical action $a_{hyp}$, environment state $s$, value function $Q$
   \ENSURE List of predicted Q-values $q_{N}$
   \STATE Initialize list of Q-values $q_{N}$
   \FOR{$i=0 \; \text{to} \; N$}
   \IF{$A_{N}[i]=a_{hyp}$}
   \STATE $q_{N}[i] \gets Q(s, a_{hyp})$
   \ELSE
   \STATE $s' \gets f_{\theta}(s, A[i])$ \COMMENT{predict one step forward}
   \STATE $s_{hyp} \gets b_{\phi}(s', a_{hyp})$ \COMMENT{predict hypothetical state}
   \STATE $q_{N}[i] \gets Q(s_{hyp}, a_{hyp})$
   \ENDIF
   \ENDFOR
\end{algorithmic}
\end{algorithm}
\subsection{How limiting are the assumptions introduced?}\label{assumptions_discussion}
It is unlikely that all assumptions will be met completely in the deployment of equivalent effect abstraction. The remainder of this work studies empirically how the performance varies as assumptions [\ref{ass:one}-\ref{ass:two}] are violated.

\textbf{Most systems in the real world have some level of stochasticity in their dynamics.} While theorem \ref{theorem:EEA_is_homomorphism} relies on assumptions of determinism, we find empirically that allowing our models to predict the mean of stochastic transitions, and then using these predictions in algorithm \ref{alg:example} enables faster value learning (see section \ref{predator_prey} and \ref{minatar}). In future work, equivalent effect abstraction could be extended to perform predictions on latent distributions or aggregate samples from such distributions \cite{hafner2019learning, yu2021learning}.

\textbf{Regarding assumption \ref{ass:three} that previous states are always computable.} In some environments like Cartpole, it is always possible to predict a valid previous state given any hypothetical action and any current state, meaning assumption \ref{ass:three} always holds. In other environments backwards state predictions are not always possible such as near a border in a 2D navigation setting. 

One strategy to avoid uncomputable hypothetical states is to select hypothetical actions that minimize violations of assumption \ref{ass:three} (i.e. actions that are plausible previous actions for many states in the environment). For example, in many video games and physics simulators there an action which is a ``no-op'', which usually leaves states unchanged. Additionally, a practitioner can choose not to collapse a portion of the action space. This is the approach we apply in section \ref{minatar} on Space Invaders and Seaquest, where we can deduce beforehand that there will be no hypothetical state-action pairs for states where the ``fire'' action is chosen. In this case, we predict the Q-values for the fire action separately and then collapse all the other action values onto the hypothetical action.

Additionally, when uncomputable states are still present, we find that allowing our models to hallucinate a previous state that could have lead to the current state allows for reasonable value predictions to be made in section \ref{minatar} (even if there is no ground truth hypothetical state that is plausible given environment dynamics). In future work, this assumption could be addressed more directly, by predicting whether a previous action can lead to a valid state transition for a given current state as in \cite{yu2021learning}. If border states are known (section \ref{gridworld}) one call fall back to full Q-learning in that particular state. 

\textbf{Lastly, our theoretical results assume that rewards only depend on the current state} and not the previous state and action (see assumption \ref{ass:two}). While this limits our applicability in optimal control settings \citep[p. 145]{liberzon2011calculus}, we posit that there are many environments where our approach is applicable, including standard benchmarks such as Atari \cite{bellemare2013arcade} or Deepmind Control \cite{tassa2018deepmind}. Note that Deepmind control would require an extension of equivalent effect abstraction into continuous action spaces. Lastly, in future work the difference in costs for actions could be estimated and corrected for. 
\section{Experiments}\label{experiments}
 Our experiments gradually build up in their complexity---with each successive experiment, we investigate how performances changes when the reliance on our assumptions is further relaxed. Where possible, we overlap with the MDP homomorphism literature---using Cartpole to overlap with \citet{van2020mdp, van2020plannable} and Predator Prey to overlap with \citet{van2020mdp}. Hyperparameters are in the Appendix. Shaded regions indicate the standard error.
\subsection{Tabular Maze}\label{gridworld}
\textbf{We begin an idealised environment where all our assumptions are satisfied}---the maze environment from \citet[p. 165]{sutton2018reinforcement}. The maze consists of $6\times9$ cells. The agent starts on the far west of the maze and must move east around walls in the middle of the environment. In this experiment, we assume a perfect model of the environment is known beforehand---which we leverage to create a homomorphic map. 
\newline
\newline
We use an open-source Q-learning \citep{watkins1992q} and environment  implementation \citep{dynamaze}. We adapt this implementation to make use of a homomorphic map as shown in algorithm \ref{alg:example}. We only apply algorithm \ref{alg:example} to non-border states. In border states, we use standard Q-learning value predictions and maintain a table of values for all border state-action pairs. 

Our approach compares favorably to Q-learning, converging much faster to the optimal policy (figure \ref{gridworld_figure}(a)). In the gridworld environment, we reduce the size of the state-action space and hence the number of Q-values (when going left is the hypothetical action) from $216$ to $67$. In this case, the number of states in the reduced action space is not reduced by the cardinality of the action space exactly. This is because of the 13 states where it is not possible to reach them by traveling left because there is a border to the right.
\begin{figure*}[htp]\label{figure:cart_pred}
     \centering
     \subfigure{\includegraphics[width=0.3\textwidth]{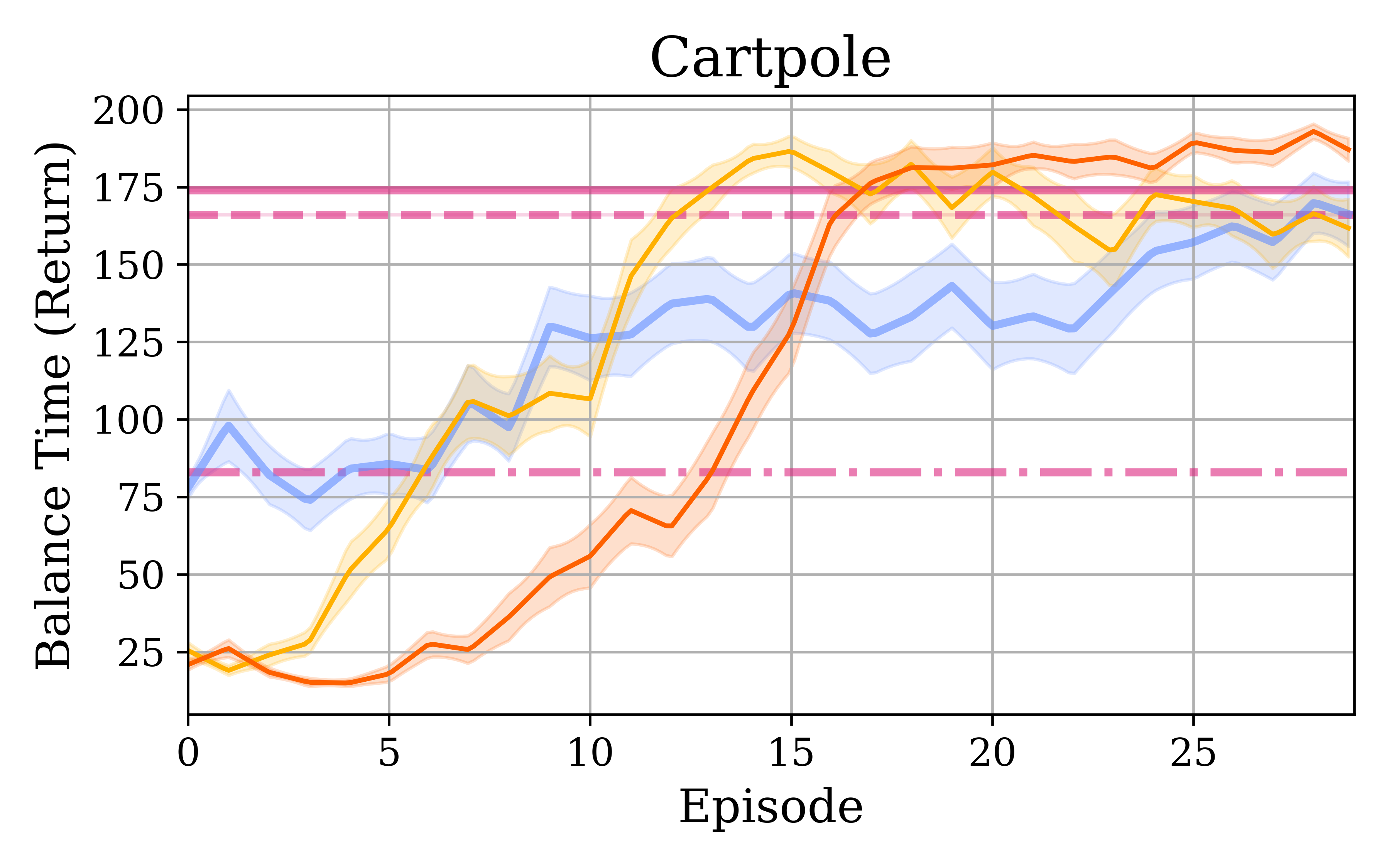}\label{fig:cartpole}}\quad
     \subfigure{\includegraphics[width=0.3\textwidth]{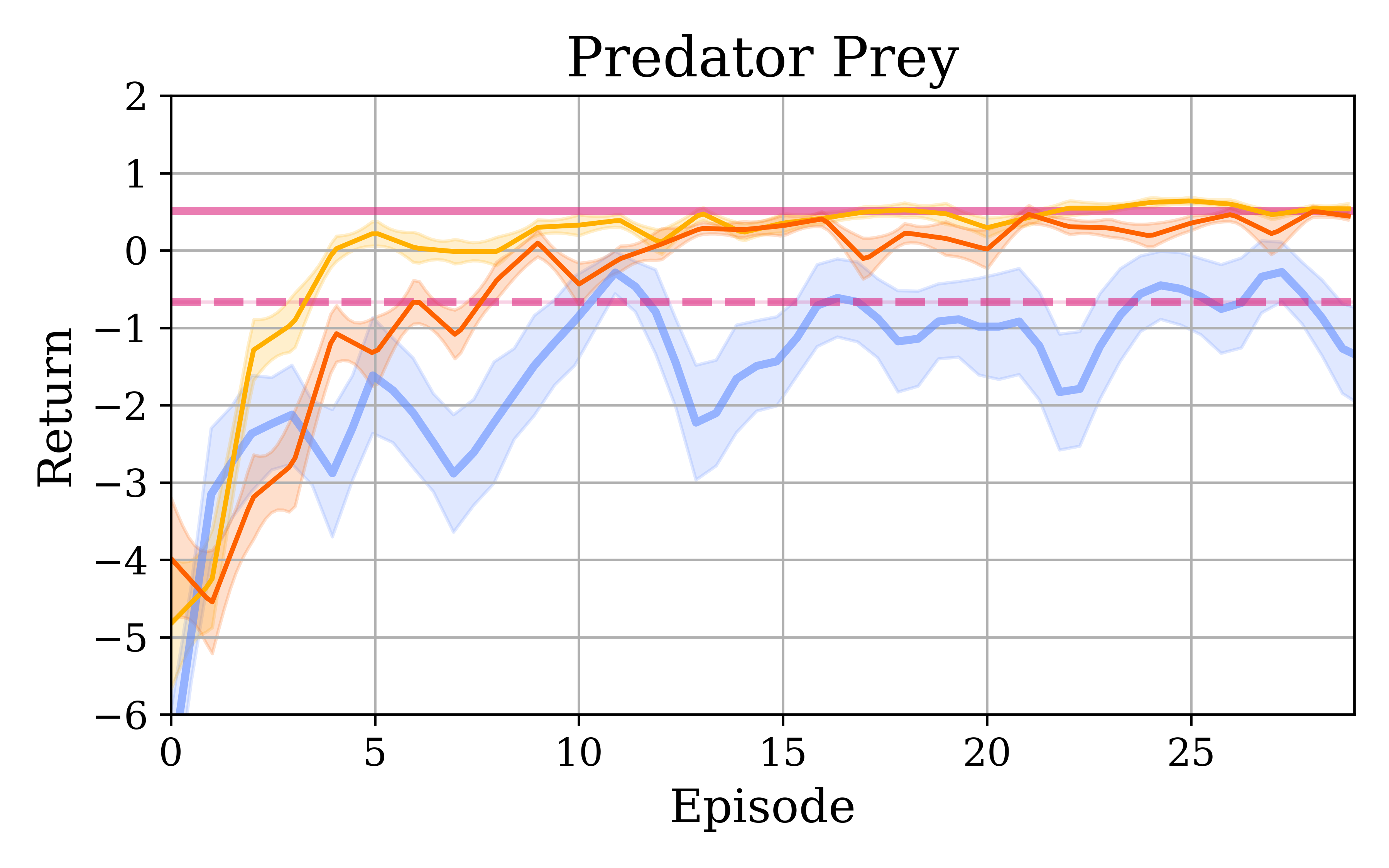}}
     \subfigure{\raisebox{0.36\height}{\includegraphics[width=0.12\textwidth]{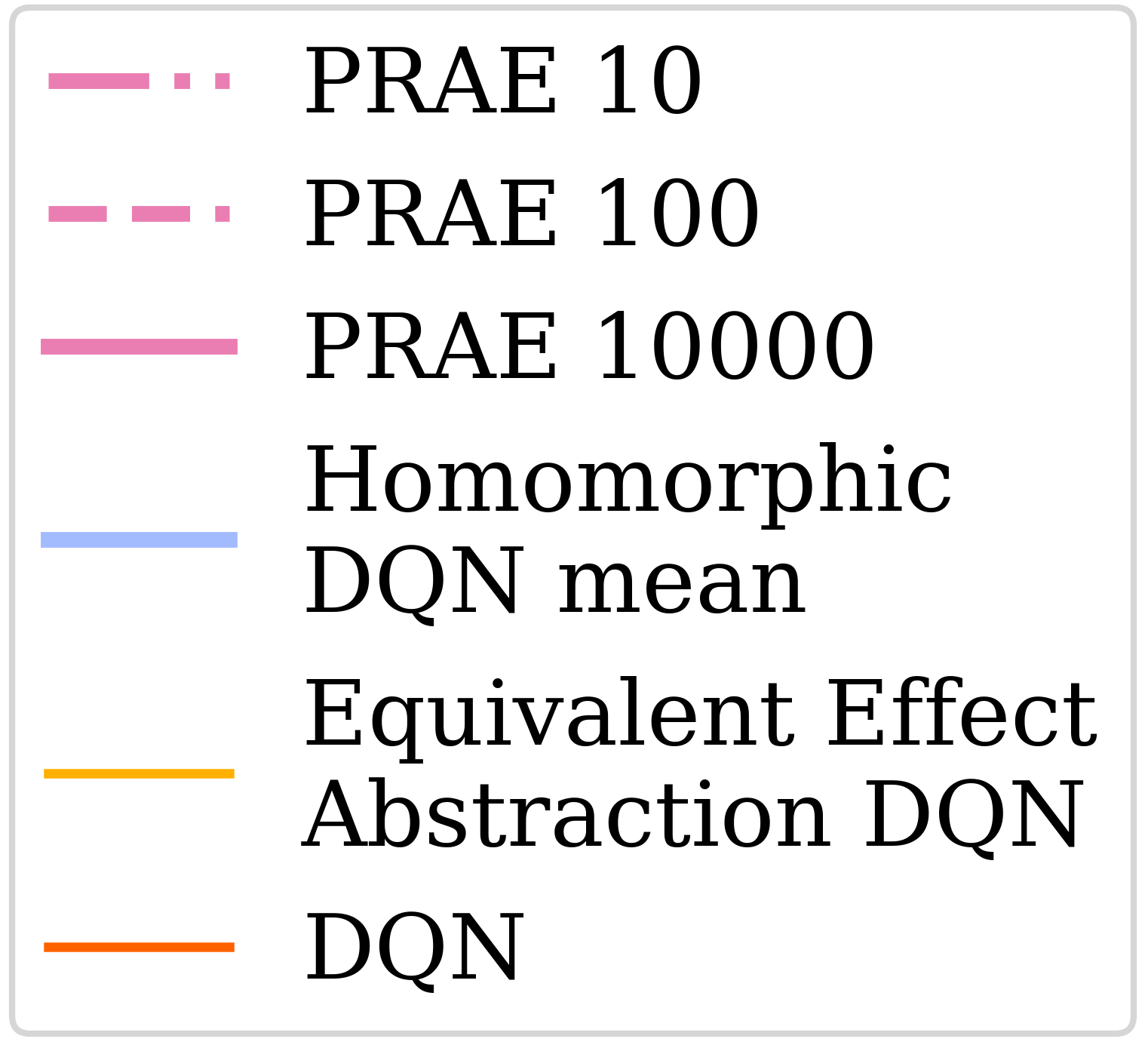}}}
    \caption{$\mathbf{(a, b)}$ Equivalent effect abstraction learns more quickly than the DQN baseline. For Cartpole, dynamics models are learned in 3 initial episodes of experience that are shown on the x-axis of the plot. For predator prey, the experience used to train the dynamics models is not included in the plot. PRAE 10 refers to using PRAE with 10 episodes of data to construct an environment model, which it then plans in. For predator prey the equivalent effect abstraction results reuse one set of pre-trained backwards and forwards models for each RL run.}\label{fig:predatorprey}
\end{figure*}
\subsection{Cartpole}\label{cartpole_section}
\textbf{Testing performance when the dynamics models are learned and therefore imperfect.} We apply equivalent effect abstraction to the Cartpole environment \citep{brockman2016openai}. In Cartpole, an agent controls a cart with a pendulum attached that it must balance upright. The states are a four-dimensional vector $(\text{position}, \text{velocity}, \text{angle}, \text{angular velocity})$ and the action space is discrete (move the cart left or right). To learn our models, we train on 3 episodes of experience at the start of training (two simple linear models, one for each action optimised with Adam \citep{kingma2015adam}). These initial training steps \emph{are} included for equivalent effect abstraction in figure \ref{fig:cartpole}.
We apply our abstraction in a DQN implementation \citep{silver2016mastering, cartpolepy}, where Q-values are only learned for an arbitrary hypothetical action.

We compare to three baselines: DQN, MDP homomorphic Networks \citep{van2020mdp} and PRAE \citep{van2020plannable}. For all experiments we use the baselines' implementations \citep{stooke2019rlpyt}. MDP homomorphic networks use weights that are equivariant to symmetries (requiring prior knowledge). PRAE trains a contrastive model to map to a latent ``plannable'' MDP, that satisfies the definitions of an MDP homomorphism. PRAE then plans in the learned abstract MDP. For Cartpole, we learned a mapping for PRAE with datasets of $[10, 100, 10000]$ episodes of random experience. We plot the converged performance of PRAE's planning algorithm. For the Cartpole and Predator Prey experiments (next section) we use 30 seeds for DQN, equivalent effect abstraction and the homomorphic DQN. For PRAE, we only use 5 seeds due to its computational expense. 
\newline
\newline
As shown in figure \ref{fig:cartpole}, in the low sample regime we improve upon both PRAE and DQN---with our approach reaching a score of 175 at around episode 12 while Q-learning takes around 16 episodes to converge---note that this improvement includes the number of episodes required to learn our mapping. The MDP homomorphic network also learns quickly, using prior knowledge of environment symmetries, but takes longer to converge.
\subsection{Stochastic Predator Prey}\label{predator_prey}
\textbf{Testing performance when dynamics models are learned and the dynamics are stochastic.} Following the seminal work of \citet{van2020mdp}, we benchmark on the Predator Prey environment, where an agent must chase a stochastically moving prey in a 2D world. The observations are $7\times7\times3$ tensors.
\newline
\newline
We train action-dependent forwards and backwards models on a dataset of transitions created by taking random actions in the environment for $10^4$ environment steps (equivalent to around $170$ episodes). Note that the environment is stochastic, so learning a perfect environment model is impossible. We compare to the baselines introduced in section \ref{cartpole_section}. For PRAE, we benchmark with 10000 and 100 episodes of data and perform planning to convergence. 
\begin{table*}
\centering
\small
\caption{Normalized scores for MinAtar 250K}
\label{table:minatar_250k}
\begin{tabular}{cccccc}
\toprule
game & actions removed & DQN & EEA & DYNA & homomorphic DQN \\
\midrule
Asterix & 4 & 0.093 ± 0.015 & \textbf{0.3 ± 0.056} & 0.05 ± 0.0058 & 0.076 ± 0.0099 \\
Breakout & 2 & 0.71 ± 0.052 & 0.77 ± 0.036 & 0.72 ± 0.046 & 0.65 ± 0.077 \\
Freeway & 2 & 0.7 ± 0.02 & 0.72 ± 0.023 & 0.037 ± 0.025 & 0.018 ± 0.012 \\
Seaquest & 4 & 0.031 ± 0.0074 & \textbf{0.11 ± 0.02} & 0.039 ± 0.0056 & 0.046 ± 0.0043 \\
Space Invaders & 2 & 0.69 ± 0.028 & 0.66 ± 0.032 & 0.79 ± 0.079 & 0.71 ± 0.061 \\
\bottomrule
\end{tabular}
\end{table*}
In Predator Prey, a policy is learned with around 5 episodes with a DQN. As a result, in this particular environment it is not practical to use equivalent effect abstraction as by the time a model can be learned the policy has already converged. However, figure \ref{fig:predatorprey}(b) shows that if a model can be obtained before policy learning, then equivalent effect abstraction can deliver an improvement over DQN---even though the environment is stochastic.
\subsection{MinAtar}\label{minatar}
\textbf{Testing performance when models are learned, dynamics are stochastic, and hypothetical states may be incomputable.} Next, we benchmark on MinAtar \citep{young19minatar}---a more accesible version of the Atari benchmark \citep{bellemare2013arcade} where observations are image-like tensors. Minatar is relatively high dimensional compared to previous sections [\ref{gridworld}-\ref{predator_prey}] and the homomorphisms literature \citep{van2020mdp}. For example, Seaquest has observations of $10\times10\times10$ tensors. Further, MinAtar has sticky actions \citep{machado2018revisiting}---meaning predicting the next state perfectly is impossible. Lastly, MinAtar games contain border states where predicting the equivalent hypothetical state is impossible.
\newline
\newline
\noindent \noindent We do not make any corrections for stochasticity or impossible to compute hypothetical border states. The only inductive bias we inject is the hard-coding of the hypothetical actions. For Asterix and Breakout, we select ``no-op'' as our hypothetical action, Freeway uses ``up'' as its hypothetical action while in Seaquest and Space Invaders we use two actions in the reduced MDP ``fire'' and ``no-op'' (because there is equivalent moving/no-op action for fire we collapse all the movement actions onto no-op and leave fire value predictions untouched). We train a U-Net style architecture \citep{ronneberger2015u} for our dynamics models (see appendix \ref{minatar_arch} for details on architecture). We train these models continuously in parallel with our Q-network, making model updates with samples from the replay buffer after each Q-learning step.
\begin{table*}
\centering
\small
\caption{Peak performance across all frames}
\label{table:minatar_peak_performance}
\begin{tabular}{cccccc}
\toprule
game & actions removed & DQN & EEA & DYNA & homomorphic DQN \\
\midrule
Asterix & 4 & \textbf{1.0 ± 0.059} & 0.85 ± 0.083 & 0.67 ± 0.047 & 0.69 ± 0.10 \\
Breakout & 2 & 1.2 ± 0.098 & 1.0 ± 0.11 & 1.0 ± 0.17 & 1.1 ± 0.11 \\
Freeway & 2 & \textbf{1 ± 0.019} & 0.96 ± 0.015 & 0.79 ± 0.022 & 0.97 ± 0.0064 \\
Seaquest & 4 & 1.1 ± 0.17 & 0.91 ± 0.18 & 1.1 ± 0.26 & 0.80 ± 0.13 \\
Space Invaders & 2 & 1.2 ± 0.090 & 1.1 ± 0.093 & 0.79 ± 0.079 & \textbf{1.4 ± 0.099} \\
\bottomrule
\end{tabular}
\end{table*}
\begin{figure}
  \begin{center}
    \includegraphics[width=0.4\textwidth]{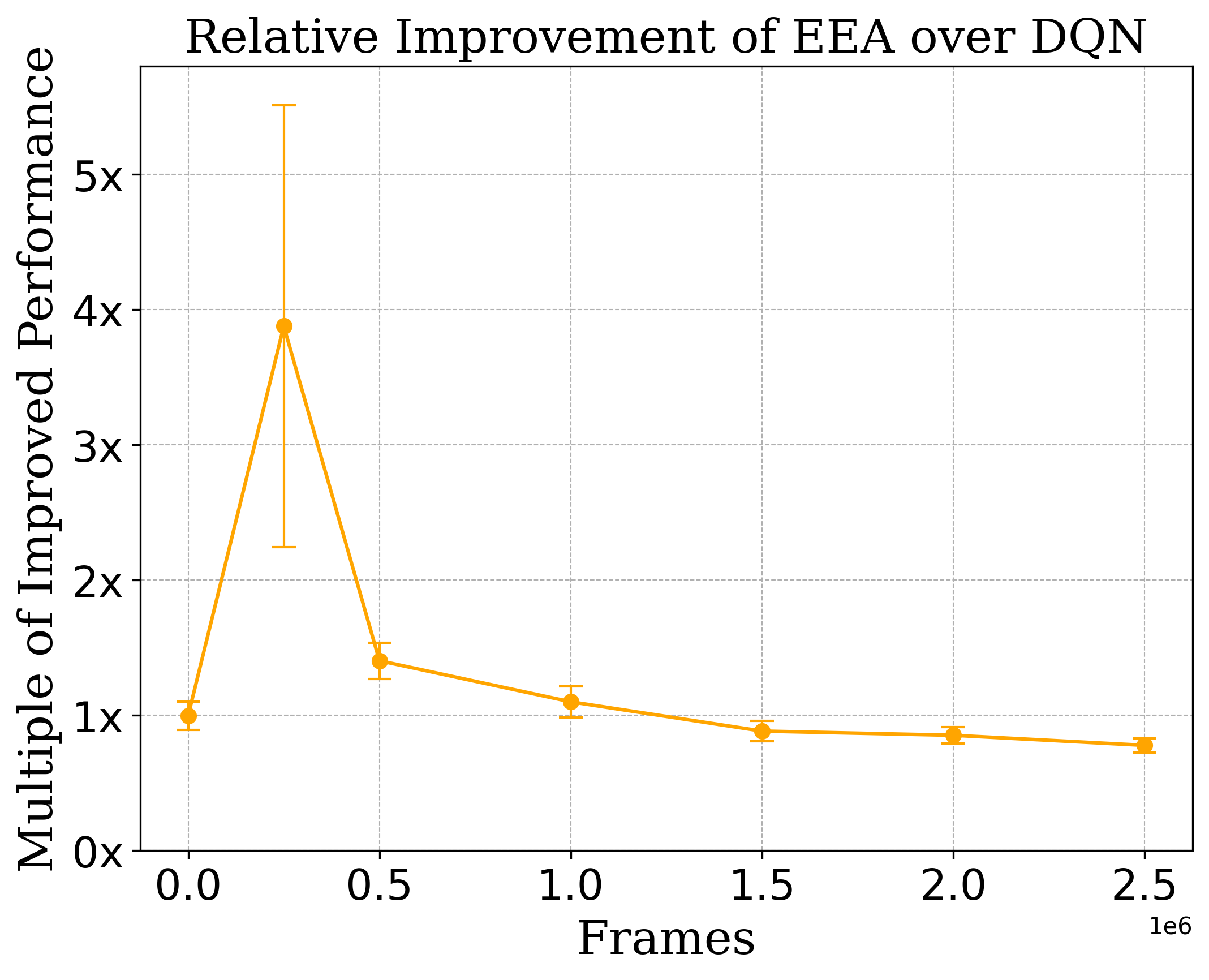}
  \end{center}
  \caption{EEA delivers an almost 4x improvement at 250k frames when averaging over all games and optimizers.}
  \label{fig:minat_relative_improvement}
\end{figure}
\newline
\newline 
We include three different baselines, a DQN implementation that we build upon \citep{young19minatar}, a homomorphic DQN \citep{van2020mdp} and a model-based Dyna-DQN agent \citep[p. 164]{sutton2018reinforcement}. The Dyna agent that augments the DQN baseline with additional experience from a learned model for a direct comparison with equivalent effect abstraction. Our Dyna agent makes forward state, reward and episode termination predictions---it then augments its value network experience with these rollouts. While other model-based baselines exist (e.g. \citet{hafner2019learning}), they are on-policy and actor-critic based making them an unfair comparison to our off-policy value-based DQN. We also note that the state of the art in discrete control is a off-policy value based agent, which outperforms all model-based approaches in the low-sample regime \citep{schwarzer2023bigger}. As a result, improving the sample efficiency of off-policy value-based methods could to advance the overall state-of-the-art.
\newline
\newline
We report normalised scores, dividing the score by the performance of a DQN after $2.5 \times 10^{6}$ frames. Inspired by \citep{kaiser2019model}, we benchmark performance at 250k frames. At 250k frames, equivalent effect abstraction outperforms all baselines on 2/5 games while matching the other best algorithm's performance on 2/5 games and only being significantly worse than than the best algorithm (Dyna) on one game. While all games have 6 actions, the number of \textit{effective} actions may be smaller. For example, in Freeway only up, down and ``no-op'' are relevant---meaning the number of actions collapsed by equivalent effect abstraction is only 2. We outperform other methods when the number of effective actions removed is large.

In table \ref{table:minatar_250k}, equivalent effect abstraction (EEA) and DQN performances are quoted when using Adam, the overall best optimizer in the low sample limit for EEA and DQN. All other methods report performance with the default RMSprop. In some environments, like Seaquest and Asterix, using Adam performs worse than RMSprop for both DQN and EEA (see figure \ref{fig:full_minatar_results} for a full breakdown). 

To disentangle the effects of specific optimizers on different games from the effect of using EEA, we ask the following: how much of an improvement does EEA give when compared to DQN when averaging over all games and optimizers? To do this we compute the ratio of EEA RMSprop performance to DQN RMSprop performance and the ratio of EEA Adam performance to DQN Adam performance for each game. After, we average these ratios over all games and plot them against training experience in figure \ref{fig:minat_relative_improvement}, showing on average \textbf{an almost 4x improvement over DQN in the low sample limit}. Error bars are standard error.
\newline
\newline
\textbf{Tradeoffs of equivalent effect abstraction in the large sample limit.} EEA maintains respectable performance as the number of samples increases but does not improve overall performance in the large sample limit. Results in table \ref{table:minatar_peak_performance} use the RMSprop optimizer for DQN and EEA as it performs best given a large sample limit. We attribute EEAs weaker performance here to violations of the assumptions we made in the theoretical development of our method. In summary, when samples are scarce, equivalent effect abstraction can improve performance. On the other hand, stochasticity, imperfect models and uncomputable hypothetical states mean that baselines overtake equivalent effect abstraction as the number of samples increases.
\section{Related Work}\label{related_work}
Work on MDP homomorphisms was initialised by \citet{ravindran2001symmetries} who developed a framework for abstraction under symmetry with theoretical guarantees. Contemporaneously, \citet{givan2003equivalence} proposed model minimisation \citep{larsen1991bisimulation}, where states could be mapped to an abstract state if their transition dynamics and rewards were indistinguishable. Related approaches explore state-action equivalences where transitions pairs lead to otherwise unrelated states but with the same probabilities \citep{asadi2019model, lyu2023scaling}. For a more general taxonomy of state-action abstractions, \citet{abel2020value} define related abstractions and their characteristics. 

In contrast to our approach, homomorphisms are often used in tasks with symmetries. \citet{van2020mdp} learn network weights that are equivariant to environment symmetries but require a practitioner to indentify symmetry groups beforehand. Similar approaches apply to environments with continuous symmetries \citep{wang2022so}. \citet{biza2019online} approach the problem of finding MDP homomorphisms using online partition iteration---predicting which partition a state should fall into given an action, and refining the partitions through splitting. 
\newline
\newline
Another class of MDP abstraction methods use bisimulation metrics (which compute distances between trajectories), often in conjuction with a contrastive loss function. For example, \citet{biza2021learning} use bisimulation to compute hidden Markov model priors to infer a reduced state space. \citet{ferns2011bisimulation, 10.5555/3020751.3020774}, condition embeddings with contrastive losses to generate invariant representations. \citet{zhang2020learning} augment an embedding encoder with an additional contrastive loss term that results in encoder distances forming an approximate bisimulation metric, which allowing agents to perform control tasks in the presence of distractors. \citet{rezaei2022continuous} formalise continuous homomorphisms with a bisimulation metric. In general, contrastive learning achieves impressive results but can be sample hungry \citep{van2020plannable}. 

Perhaps the most relevant related approach is the use of afterstates \citep[p. 136]{sutton2018reinforcement}. Afterstates shift an MDP out of phase with conventional transitions, creating states in-between the effect of a policy's action and the reaction of the environment. Applications of afterstates have been constrained to board games \citep{tesauro1995temporal} and focus on dealing with the stochasticity of an opponent rather than improved sample efficiency \citep{antonoglou2021planning}. Also related, \citet{misra2020kinematic} deploy a similar concept, finding states with equivalent transition dynamics by grouping together state-action pairs that will pass through or have passed through same state with the same previous action in block MDPs, where a discretisation of the state space needs to be learned. Our approach goes beyond this, looking at abstracting state-action pairs that do not necessarily have equivalent dynamics for the same actions but instead for different actions. Also related, \cite{grinsztajn2022better} learn to explore by avoiding action sequences that lead to the same next state.
\newline
\newline
More broadly, model-based RL has enabled superhuman performance in Atari \citep{hafner2020mastering}, but the number of planning steps required to learn a policy is still large. Backwards models have been proposed to improve representation learning \citep{yu2021learning, yu2021playvirtual}, which would be interesting if integrated with equivalent effect abstraction. 
\section{Future Work}\label{limitations}
While many of the recent breakthroughs are value based (e.g. \citep{silver2017mastering, badia2020agent57}), actor-critic approaches \citep{schulman2017proximal, mnih2016asynchronous} are often the natural choice for control tasks. Embedding equivalent effect abstraction into actor-critic architectures is a fruitful avenue for future research. It is also conceivable to formulate equivalent effect abstraction within a continuous action space by simply discretising the action space \citep{banino2021coberl}. Lastly, equivalent effect abstraction could also be integrated into existing homomorphic MDP methods that rely on symmetries to further reduce the size of the abstract state-action space \citep{van2020mdp}.

\section{Conclusion}\label{conclusion}
Equivalent effect abstraction is a simple method that reduces the size of a state-action space and requires no prior knowledge of environment symmetries. We show equivalent effect abstraction can improve the sample efficiency of policy learning in tabular environments, control tasks with continuous state spaces, and stochastic deep RL environments. An exciting next step would be to integrate equivalent effect abstraction into model-based RL \citep{hafner2020mastering, yu2021learning} to improve planning efficiency.
\bibliography{example_paper.bib}
\bibliographystyle{apalike}
\newpage
\appendix
\section{Appendix}
\subsection{Impact Statement}
We do not believe that EEA introduces any novel dangers or concerns in terms of broader impact. Like all reinforcement learning algorithms, our agents likely do not generalise very well and could behave unexpectedly when deployed in real world environments. We suggest extreme caution when using any RL algorithm for consequential applications.

\subsection{Theoretical Results}
\textbf{Theorem \ref{theorem:EEA_is_homomorphism}} \textit{Given assumptions [\ref{ass:one}-\ref{ass:three}], the MDP mapping provided by equivalent effect abstraction is an MDP homomorphism.} 
\begin{proof}\label{proof:homomorphism}
We can use our definition of state and action mappings in the left hand side.
\begin{equation}\label{eqn:plug_in_transition}
        \mathcal{P}(b_{\phi}(f_{\theta}(s, a), a_{hyp}), a_{hyp}, s')=\sum_{s''\in [s']_{B_{h}|\mathcal{S}}}\mathcal{P}(s, a, s'')
\end{equation}
Where we use the fact that the dynamics of the reduced MDP are equal to that of the full MDP (except with a reduced action space), meaning $\bar{\mathcal{P}}=\mathcal{P}$. Using assumption \ref{ass:three}, which states that there exists a hypothetical state action pair $(s_{hyp}, a_{hyp})$ that transitions into $s'$, equation \ref{eqn:plug_in_transition} reduces to the following.
\begin{equation}\label{eqn:transition_proof_penultimate}
    \mathcal{P}(s_{hyp}, a_{hyp}, s')=\sum_{s''\in [s']_{B_{h}|\mathcal{S}}}\mathcal{P}(s, a, s'')
\end{equation}
Assumption \ref{ass:one} means our MDP assumes deterministic dynamics, which means allows equation \ref{eqn:transition_proof_penultimate} to be rewritten without the summation.
\begin{equation}\label{eqn:transition_proof_final}
    \mathcal{P}(s_{hyp}, a_{hyp}, s') = \mathcal{P}(s, a, s')
\end{equation}
Given our assumptions, both sides of \ref{eqn:transition_proof_final} are equal to one if $(s, a, s')$ is a valid environment transition, otherwise both sides are equal to zero. Finally, one can see from the defintion \ref{definition:EEA}, that the following is true by construction.
\begin{align}
    \bar{\mathcal{R}}(s_{hyp}, a_{hyp}, s')
    =\mathcal{R}(s, a, s')
\end{align}
As a result, equivalent effect abstraction obeys both the transition and reward function requirements of an MDP homomorphism.
\end{proof}
\subsection{Hyperparameter Search}\label{hyperparam}
Below we show the hyperparameters swept through for homomorphic MDP, DQN and Equivalent Effect Abstraction agents, broken down by environment. The PRAE architectures were generally kept the same as the hyperparameters provided in \citep{van2020plannable}, with the exception of the learning rate which we evaluate at $0.0001, 0.001$ and $0.1$
\subsubsection{Sutton and Barto Tabular Gridworld}
We use the hyperparameters specified in \citep[p. 165]{sutton2018reinforcement}: namely, Learning Rate$= 0.1$, $\gamma = 0.95$ and $\epsilon = 0.1$

\subsubsection{Cartpole}
\begin{table}[h!]
\label{sample-table}
\begin{center}
\caption{Hyperparameters swept through for the Cartpole environment. Learning rate decay refers to decaying the learning rate by a factor of ten at after a specified number of episodes have elapsed.}

\begin{tabular}{ll}                                       
\toprule
\multicolumn{1}{c}{\bf Hyperparameter}  &\multicolumn{1}{c}{\bf Values}          \\ \midrule \\ 
Learning Rate        & $
0.00001, 0.0001, 0.001, 0.01$ \\ 
$\epsilon$ decay schedule             & No decay, exponential $\tau = \frac{-1}{200}$ \\ 
$\gamma$             & $0.8, 0.99$\\
Activation & ReLU, tanh \\
Learning Rate decay & No decay, 5, 10, 15, 20 \\
\bottomrule
\end{tabular}
\end{center} 
\end{table}                  

\textbf{Homomorphic MDP} best hyperparameters: Learning Rate = $0.001$, $\epsilon$ decay schedule = No decay, $\gamma = 0.8$, activation = tanh, Learning Rate decay = No decay
\newline
\newline
\textbf{Equivalent Effect Abstraction} best hyperparameters: Learning Rate = $0.001$, $\epsilon$ decay schedule = No decay, $\gamma = 0.8$, activation = tanh, Learning Rate decay = 10
\newline
\newline
\textbf{DQN} best hyperparameters: Learning Rate = $0.001$, $\epsilon$ decay schedule = No decay, $\gamma = 0.8$, activation = tanh, Learning Rate decay = 15
\subsubsection{Stochastic Predator Prey}

\begin{table}[h!]
\begin{center}
\caption{Hyperparameters swept through for the Predator Prey environment.}
\begin{tabular}{ll}
\toprule
   \multicolumn{1}{c}{\bf Hyperparameter}  & \multicolumn{1}{c}{\bf Values} \\ \midrule \\
Learning Rate     & 0.0001, 0.001, 0.01 \\
$\gamma$ & 0.8, 0.99 \\
\bottomrule
\end{tabular}
\end{center}
\end{table}
\textbf{Homomorphic MDP} best hyperparameters: Learning Rate = $0.001$, $\gamma = 0.99$
\newline
\newline
\textbf{Equivalent Effect Abstraction} best hyperparameters: Learning Rate = $0.01$, $\gamma = 0.8$
\newline
\newline
\textbf{DQN} best hyperparameters: Learning Rate = $0.001$, $\gamma = 0.99$
\newline
\newline
\subsubsection{Minatar}\label{minatar_hyperparams}
All hyperparameters were kept at the tuned DQN values provided in \cite{young19minatar} except for the optimizer for DQN and EEA, which we experimented with using Adam and RMSprop. For the homomorphic DQN, the symmetry groups breakdown as follows:
\newline
\newline
\begin{table}[h]
\centering
\captionsetup{justification=centering}
\caption{Symmetry groups used for the homomorphic DQN \citep{van2020mdp} on MinAtar.}
\label{your_label}
\begin{tabular}{ll}
\toprule
\textbf{Game} & \textbf{Symmetry Group} \\
\midrule
Asterix & P4 \\
Breakout & R2 \\
Freeway & R2 \\
Seaquest & P4 \\
Space Invaders & R2 \\
\bottomrule
\end{tabular}
\end{table}
\newline
\newline
The groups for the homomorphic DQN were chosen by looking at gameplay footage and estimating what symmetry groups would be most useful.
\newline
\newline
For the model learning (in both EEA and Dyna), we found a learning rate of 0.0001 with an Adam optimizer was optimal. For Dyna, we tried 1 and 3 step rollouts as well as different replay ratios and found a replay ratio of 1 and rollout length of 1 to have the highest mean overall scores.
\newline
\newline
The returns provided by training runs for MinAtar be noisy---to make the performance of the algorithms at different numbers of frames clear we pause each training run and perform 30 evaluations to get the performances we quote in the tables and figures in the main text.
\subsection{Model Architectures}
\subsubsection{Cartpole}
\begin{lstlisting}[language=iPython, caption={Homomorphic MDP Network \citep{van2020mdp}}]    
BasisLinear*(repr_in=4, channels_in=1, repr_out=2, channels_out=64)
ReLU() / tanh()
BasisLinear(repr_in=2, channels_in=64, repr_out=2, channels_out=64)
ReLU() / tanh() 
BasisLinear(repr_in=2, channels_in=64, repr_out=2, channels_out=1)
\end{lstlisting}
*$\texttt{BasisLinear}$  refers to the symmeterised layers used in \citep{van2020mdp} to create homomorphic networks. This network is identical to the Cartpole network presented in that paper, but with only one output head that outputs state-action values.

\begin{lstlisting}[language=iPython, caption={Value Network architecture for DQN and Equivalent Effect Abstraction}]    
Linear(input_size=4, output_size=1024)
tanh()
Linear(input_size=1024, output_size=1024)
tanh()
Linear(input_size=8, output_size=1024)
tanh()
Linear(input_size=1024, output_size=2)
\end{lstlisting}

\begin{lstlisting}[language=iPython, caption={Transition Model Architecture for Equivalent Effect Abstraction}]    
Linear(input_size=2, output_size=2)
\end{lstlisting}

\begin{lstlisting}[language=iPython, caption={PRAE Architectures \citep{van2020plannable}}]    
# state encoder
Linear(input_size=4 ,output_size=64)
ReLU()
Linear(input_size=64, output_size=32)
ReLU() 
Linear(input_size=32, output_size=50)
#action encoder
Linear(input_size=54 ,output_size=100)
ReLU()
Linear(input_size=100, output_size=2)
# reward prediction network
Linear(input_size=50 ,output_size=64)
ReLU()
Linear(input_size=64, output_size=1)
\end{lstlisting}

\subsubsection{Predator Prey}
\begin{lstlisting}[language=iPython, caption={Homomorphic MDP Network \citep{van2020mdp}}]    
BasisConv2d(repr_in=1, channels_in=1, repr_out=4, channels_out=4, 
filter_size=(7,7), stride=2, padding=0)
ReLU()
BasisConv2d(repr_in=4, channels_in=4, repr_out=4, channels_out=8, 
filter_size=(5,5), stride=1, padding=0)
ReLU() 
GlobalMaxPool()
BasisLinear(repr_in=4, channels_in=8, repr_out=4, channels_out=128)
ReLU() 
BasisLinear(repr_in=4, channels_in=8, repr_out=4, channels_out=128)
ReLU() 
BasisLinear(repr_in=4, channels_in=128, repr_out=5, channels_out=1)
\end{lstlisting}
This is again the same network used in \citet{van2020mdp}, albeit with a different output head.
\begin{lstlisting}[language=iPython, caption={Value Network Architecture for DQN and Equivalent Effect Abstraction}]    
Linear(input_size=441, output_size=1024)
ReLU() 
Linear(input_size=1024, output_size=8)
ReLU() 
Linear(input_size=8, output_size=1024)
ReLU() 
Linear(input_size=1024, output_size=5)
\end{lstlisting}

\begin{lstlisting}[language=iPython, caption={Transition Model Architecture for Equivalent Effect Abstraction}]    
Linear(input_size=882, output_size=512)
ReLU() 
Linear(input_size=512, output_size=8)
ReLU() 
Linear(input_size=8, output_size=512)
ReLU() 
Linear(input_size=512, output_size=441)
\end{lstlisting}

\
\begin{lstlisting}[language=iPython, caption={PRAE Architectures \citep{van2020plannable}}]    
# state encoder
Linear(input_size=441, output_size=64)
ReLU()
Linear(input_size=64, output_size=32)
ReLU() 
Linear(input_size=32, output_size=50)
#action encoder
Linear(input_size=54, output_size=100)
ReLU()
Linear(input_size=100, output_size=2)
# reward prediction network
Linear(input_size=50, output_size=64)
ReLU()
Linear(input_size=64, output_size=1)
\end{lstlisting}
\subsubsection{MinAtar}\label{minatar_arch}
\begin{lstlisting}[language=iPython, caption={Homomorphic MDP Network \citep{van2020mdp}}]
BasisConv2d(repr_in=1, channels_in=4, repr_out=2, channels_out=16,
filter_size=(3,3), stride=1, padding=0)
ReLU()
Linear(input_size=1408, output_size=256)
ReLU()
Linear(input_size=256, output_size=6)
\end{lstlisting}

\begin{lstlisting}[language=iPython, caption={Value Network Architecture for DQN and Equivalent Effect Abstraction}]    
Conv2d(channels_in=4, channels_out=16, filter_size=(3,3), stride=1, padding=0)
ReLU()
Linear(input_size=1024, output_size=128)
ReLU()
Linear(input_size=128, output_size=6)
\end{lstlisting}

\begin{lstlisting}[language=iPython, caption={MinAtar Environment Model U-net}]    
Linear(input_size=10 * 10 * num_channels + num_actions, output_size=800)
ReLU()
Linear(input_size=800, output_size=10 * 10 * num_channels)
ReLU()
Linear(input_size=10 * 10 * num_channels, output_size=10 * 10 * num_channels)
\end{lstlisting}
Where there is a skip connection from the input to output layer.
\begin{lstlisting}[language=iPython, caption={MinAtar Reward Model}]    
Linear(10 * 10 * self.num_channels + num_actions, 500)
Linear(input=500, output=10 * 10 * self.num_channels)
Linear(input=10 * 10 * self.num_channels, output=2)
\end{lstlisting}
\clearpage

\newgeometry{left=2cm,right=2cm,top=1.5cm,bottom=2cm} 
\section{Full MinAtar Results}\label{full_minatar_results}
\begin{figure*}[htp]
    \centering
    \subfigure{\includegraphics[width=0.3\textwidth]{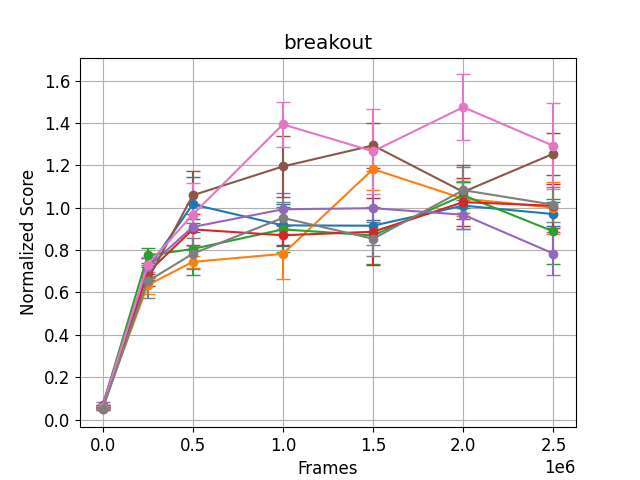}}
    \subfigure{\includegraphics[width=0.3\textwidth]{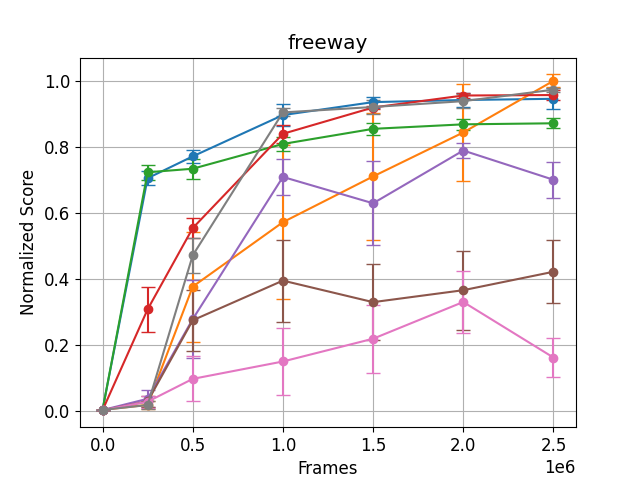}}
    \subfigure{\includegraphics[width=0.3\textwidth]{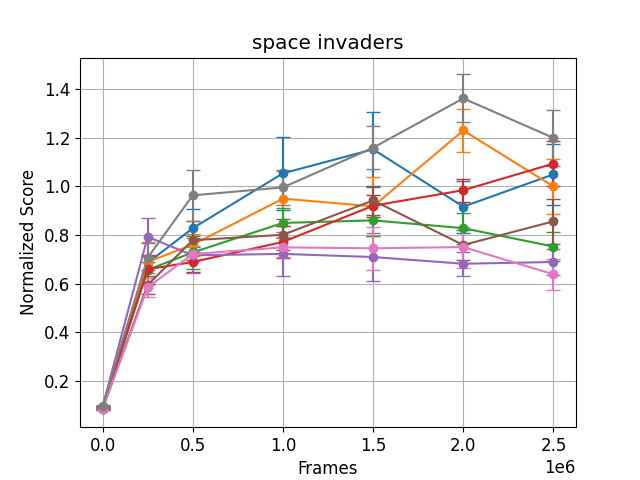}}
    \subfigure{\includegraphics[width=0.3\textwidth]{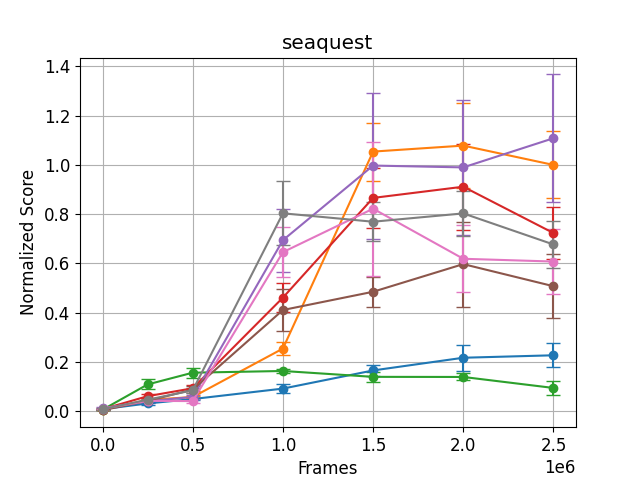}}
    \subfigure{\includegraphics[width=0.3\textwidth]{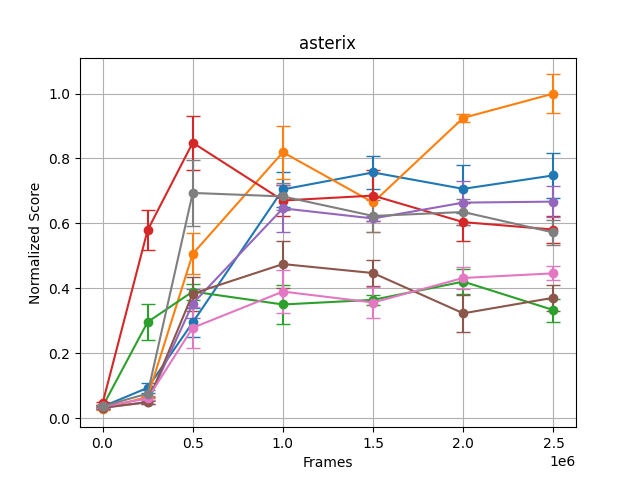}}
    \subfigure{\includegraphics[width=0.28\textwidth]{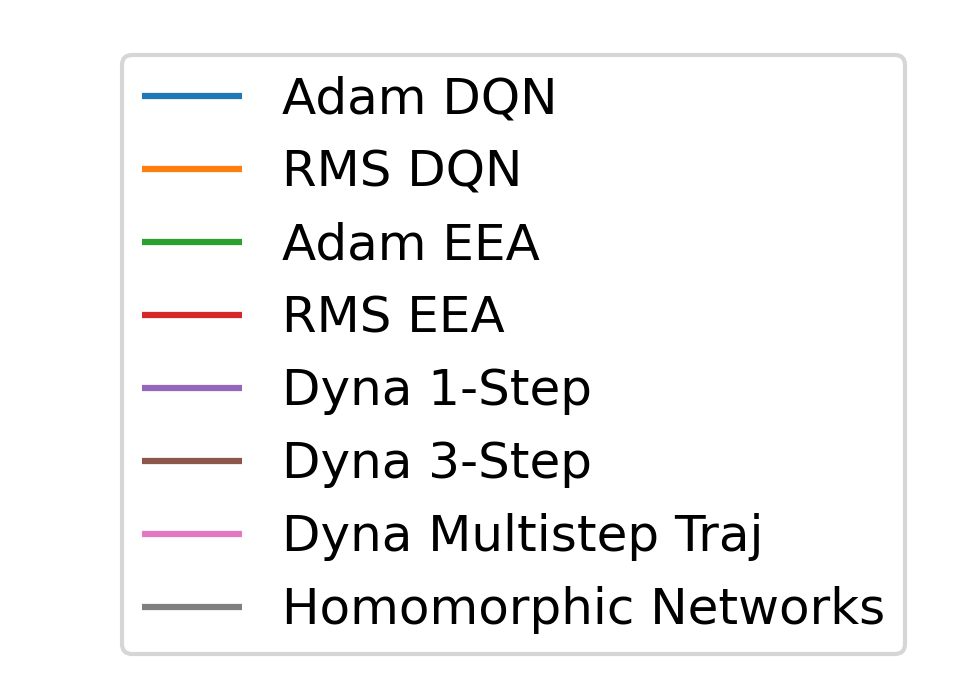}}
    \caption{Full results for all different configurations of equivalent effect abstraction and the other baselines. Equivalent effect abstraction \newline with the Adam optimizer generally performs best in the low sample limit.}
    \label{fig:full_minatar_results}
\end{figure*}
\restoregeometry

\end{document}